\documentclass[]{style}

\usepackage[toc,page,header]{appendix}

\usepackage{natbib}

\usepackage{CJKutf8}

\usepackage{xargs}  

\usepackage{todonotes}  

\usepackage{multirow}

\usepackage{cleveref}

\usepackage{amsmath}
\usepackage{dsfont}

\usepackage{svg}

\usepackage{mathrsfs}
\usepackage{adjustbox}
\usepackage{multirow}
\usepackage{multicol}
\usepackage{tcolorbox}
\usepackage{changepage}
\usepackage{enumitem}
\usepackage{graphicx}
\usepackage{amssymb}
\usepackage{xcolor}
\usepackage{float}
\usepackage{multirow}
\usepackage{threeparttable}
\usepackage{graphicx}
\usepackage{subcaption}
\usepackage{algorithm}
\usepackage{algpseudocode}
\usepackage{wrapfig}
\usepackage{epsfig}
\usepackage{overpic}
\usepackage{diagbox}
\usepackage{fontawesome}
\usepackage{bbding}
\usepackage{wasysym}
\usepackage{utfsym}
\usepackage{mathtools}
\usepackage[table]{xcolor}  % loads xcolor with table support
\usepackage{colortbl}       % gives \columncolor, \rowcolor
\usepackage{tabularx} % Add to your preamble
\usepackage{makecell} % Add to your preamble
\usepackage[dvipsnames]{xcolor}

\newcolumntype{g}{>{\columncolor{gray!10}}c} % gray background

\definecolor{catgray}{gray}{0.9}
\definecolor{skyblue}{rgb}{0.53,0.81,0.92} % sky blue

\colorlet{skyblue!30}{skyblue!30!white} % 30% skyblue, 70% white

\definecolor{customblue}{RGB}{70,130,180}  % This is equivalent to rgb(70,130,180)

\newtcolorbox{evolbox}[2][]{%
  enhanced,
  colframe=customblue,
  colback=white,
  coltitle=white,
  rounded corners,
  boxrule=1pt,
  titlerule=0pt,
  toptitle=1mm,
  bottomtitle=1mm,
  fonttitle=\bfseries,
  width=#2\textwidth, % This takes the second parameter as the width fraction
  #1
}
\usepackage{minitoc}
\usepackage{url}

\PassOptionsToPackage{table,xcdraw}{xcolor}
\usepackage{titletoc}
\usepackage{placeins}
\usepackage{pifont}

\definecolor{RowBlue}{HTML}{E9F2FB}
\definecolor{RowRed}{HTML}{F9EAEA}
\definecolor{Top1}{HTML}{50DB4B} % 深绿
\definecolor{Top2}{HTML}{A5FFA2} % 中绿
\definecolor{Top3}{HTML}{D9FFD9} % 浅绿
\definecolor{Sub1}{HTML}{EAB8B8}
\definecolor{Sub2}{HTML}{E4E4E4}

\renewcommand{\emph}[1]{\textit{#1}}

\usepackage[most]{tcolorbox}
\tcbuselibrary{theorems,skins,breakable}

\newtcbtheorem[
  number within=section,
  crefname={Theorem}{Theorems}
]{theorem}{Theorem}{
  enhanced,
  breakable,
  colback=hitblue!3,
  colframe=hitblue,
  boxrule=0.6pt,
  arc=8pt,
  left=6pt,
  right=6pt,
  top=6pt,
  bottom=6pt,
  colbacktitle=hitblue,
  coltitle=white,
  fonttitle=\sffamily\bfseries,
  lefttitle=4pt,
  righttitle=4pt,
  toptitle=1mm,
  bottomtitle=0.4mm,
  boxed title style={
    sharp corners,
    boxrule=0pt
  },
  before skip=6pt,
  after skip=6pt
}{thm}

\title{MirrorLA: Reflecting Feature Map for \\ Vision Linear Attention}
\headtitle{MirrorLA: Reflecting Feature Map for Vision Linear Attention}

\author[1,2]{Weikang Meng}
\author[1]{Liangyu Huo}
\author[3]{Yadan Luo}
\author[1,2]{Yaowei Wang}
\author[2]{Yingjian Li}
\author[1,\text{$\dagger$}]{Zheng Zhang}

\affiliation[1]{Harbin Institute of Technology, Shenzhen}
\affiliation[2]{{Pengcheng Laboratory}}
\affiliation[3]{UQMM Lab, University of Queensland}
\contribution[\text{$\dagger$}]{Corresponding author}

\preprint{Preprint version.}
\abstract{
Linear attention significantly reduces the computational complexity of Transformers from quadratic to linear, yet it consistently lags behind softmax-based attention in performance. We identify the root cause of this degradation as the non-negativity constraint imposed on kernel feature maps: standard projections like ReLU act as ``passive truncation'' operators, indiscriminately discarding semantic information residing in the negative domain. We propose MirrorLA, a geometric framework that substitutes passive truncation with active reorientation. By leveraging learnable Householder reflections, MirrorLA rotates the feature geometry into the non-negative orthant to maximize information retention. Our approach restores representational density through a cohesive, multi-scale design: it first optimizes local discriminability via block-wise isometries, stabilizes long-context dynamics using variance-aware modulation to diversify activations, and finally, integrates dispersed subspaces via cross-head reflections to induce global covariance mixing. MirrorLA achieves state-of-the-art performance across standard benchmarks, demonstrating that strictly linear efficiency can be achieved without compromising representational fidelity.
}

\correspondence{\email{darrenzz219@gmail.com}}

\begin{document}

\maketitle

% efficient attention -> LA

% kernel design; hybrid -> neglect root cause feature map passively truncate due to non-negative -> "dead dimension" -> resort to  isometry geometry -> householder actively reorient -> maximum feature preservation;
% Beyond single-head reflection effective for local details-> long-sequence SR .. -> diversification (activation, head)
% (1) Additionally,  to gaurantee long-sequence high-aligned tokens not mapped to same orthant -> variance conditioned shift "angle" -> diversify...
% (2) as stated in ... "head-head", diversify covariance mixing...

% long-sequence perf
\section{Introduction}
Transformer-based models \citep{transformer, vit} have become the dominant paradigm for modeling long-range dependencies in vision and language tasks. 
Their success largely stems from the self-attention mechanism, which aggregates global context via softmax-normalized dot-product similarity. 
Despite its effectiveness, softmax attention suffers from a quadratic computational complexity $O(N^2d)$ relative to the sequence length $N$, making it prohibitively expensive for long sequences and high-resolution visual inputs.

Linear Attention (LA) \citep{linearattn, minimax, polaformer} addresses this limitation by reformulating softmax attention using kernel feature maps that enable associative reordering of matrix multiplications.
This reformulation shifts the computation order from $\operatorname{Softmax}(\mathbf{QK^\top})~\mathbf{V}$ to $\phi(\mathbf{Q})~(\phi(\mathbf{K})^\top V)$.
Consequently, the complexity is reduced to $O(Nd^2)$, maintaining the benefits of global context modeling while ensuring computational efficiency for resource-intensive tasks.
However, despite extensive research, linear attention methods consistently underperform their softmax-based counterparts in practice, especially in vision tasks. Closing this performance gap remains a central open problem.

Previous approaches primarily attributes this gap to the design of kernel feature maps $\phi(\cdot)$. To ensure numerical stability and avoid gradient explosion, most linear attention formulations impose a non-negativity constraint on feature maps, commonly enforced via \textit{axis-aligned activations} such as $\operatorname{ReLU}$ \citep{efficientvit,flatten}, $1+\operatorname{ELU}$ \citep{linearattn}, or exponential mappings \citep{lognormalattention}. While effective for stability, these kernels inherently discard negative feature components, triggering a ``dead dimension'' crisis that severely limits representational capacity.
Some efforts attempt to mitigate the information loss caused by strict feature mapping, for instance, Hedgehog \citep{hedgehog} employs MLP-$\exp$ based mappings, while Polaformer \citep{polaformer} and Nalaformer \citep{nalaformer} explore sign-decomposition and cosine similarity. However, they often rely on expanding the feature dimension $d$ by $k$ times, which paradoxically raise the computational cost to $\mathcal{O}(k^2d^2)$ and undermines the efficiency advantage of LA.

\begin{figure}
    \centering
    \includegraphics[width=0.55\linewidth]{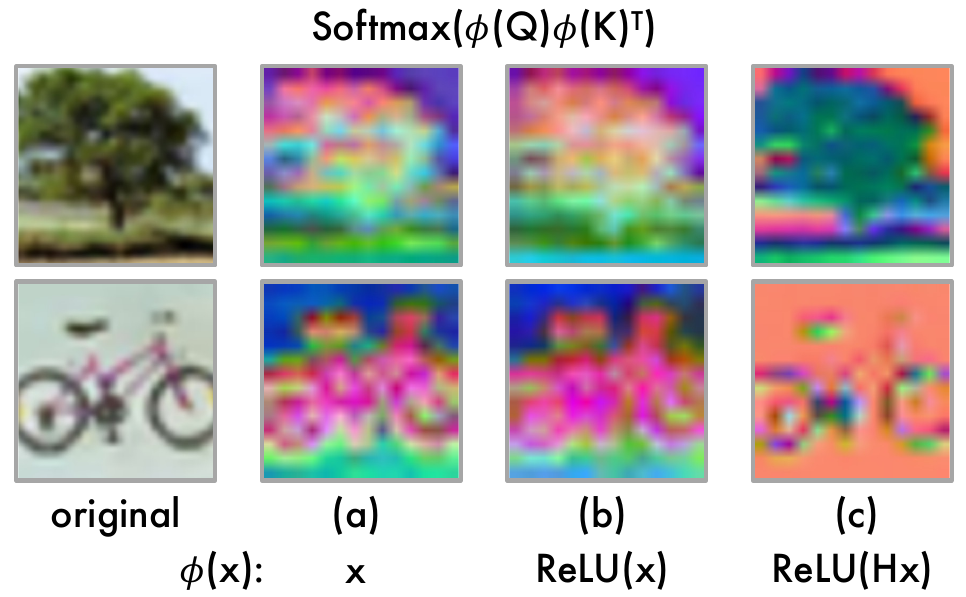}
    \caption{\textbf{PCA Visualization of Feature Topology.} We visualize $\operatorname{Softmax}(\phi(\mathbf{Q})\phi(\mathbf{K})^\top)$ under constant normalization across three paradigms: 
(a) \textit{Vanilla Attention}, $\phi(\mathbf{x})=\mathbf{x}$ or $\mathbf{Hx}$; 
(b) \textit{Passive Truncation} $\phi(\mathbf{x})=\operatorname{ReLU}(\mathbf{x})$ results in ``dead'' dimension and information loss; 
(c) \textit{Active Reorientation}, $\phi(\mathbf{x})=\operatorname{ReLU}(\mathbf{Hx})$ employs an isometric Householder reflection. This aligns informative features with the positive orthant, recovering the rich structural details lost in (b).}
% In (a), $\phi(\mathbf{x})=\mathbf{x}$ or $\mathbf{Hx}$ is mathematically equivalent, since $\langle \mathbf{Hq}, \mathbf{Hk} \rangle = \langle \mathbf{q}, \mathbf{k} \rangle$.}
    \label{fig:pca}
\end{figure}

We argue that the such a limitation lies not in the non-negativity constraint itself, but in the rigidity of its enforcement. 
Crucially, attention similarity is based on the inner product $\langle \mathbf{q}, \mathbf{k}\rangle$, which is rotation invariant. Therefore, preserving similarity does not require preserving the original coordinate axes. 
Instead of \textit{passively} truncating negative components, we hypothesize that one can  \textit{actively} reorient the feature space such that rotaing information-rich directions into the non-negative orthant before truncation. 
% This transforms the non-negativity constraint from an indiscriminate filter into a controllable geometric operation.

% Current linear attention designs heavily rely on non-negative activation functions, such as $\operatorname{ReLU},\exp$, to satisfy the non-negativity constraint required for similarity. 
% However, we argue that the core limitation of existing linear attention methods lies in the \textbf{passive truncation} of non-negative feature maps. Crucially, preserving similarity does not require preserving coordinate axes. 
% Instead of passively truncating negative components via fixed axis-aligned activations, we propose to actively reorient the feature space such that information-rich directions are aligned with the non-negative orthant before truncation. 
% This transforms the non-negativity constraint from an indiscriminate filter into a controllable geometric operation.

To validate this hypothesis, we conduct a controlled PCA visualization of feature maps in Fig. \ref{fig:pca}, comparing standard ReLU truncation against a \textit{rotation-augmentation} mapping. We implement this rotation via a Householder reflection $\mathbf{H}$, chosen for its strict \textit{isometric} \citep{householderisometry} property ($\langle \mathbf{Hq}, \mathbf{Hk} \rangle = \langle \mathbf{q}, \mathbf{k} \rangle$). A critical observation emerges: compared to the standard attention (a), the ReLU feature map (b) results in a \textit{blurred} and sparse approximation, showing that passive truncation indiscriminately discards the fine-grained contrast details.
In contrast, the reflected map (c) produces a distribution that is visually \textit{distinct} yet structurally rich. This confirms that the information loss in LA is not inevitable; by transforming the non-negativity constraint from a destructive filter into a controllable geometric operation, we can restore the latent structural integrity typically discarded by axis-aligned clipping.

Motivated by these observations, we propose \textbf{MirrorLA}, a linear attention framework that shifts the paradigm from \textit{passive clipping} to \textit{active reorientation}. MirrorLA adopts a hierarchical design based on learnable Householder reflections to optimize feature topology at local, dynamic and global levels.
At the \textit{single-head} level, MirrorLA enables precise geometric control of full-rank rotation by decomposing high-dimensional feature space into disjoint 2D subspaces. We then apply strictly isometric Householder reflections within each block to align information-bearing directions into the non-negative orthant without altering pairwise distances. To avoid mirroring features into overly narrow subspace and trigger feature collapse in long-context modeling, we make the reflection \textit{variance-conditioned}: the reflection angle is coupled to the block-wise signal variance, inducing a dynamic perturbation that yields a theoretical lower bound on activation diversity. This mitigates the dead-zone effect, where clustered tokens in low-variance regimes are indiscriminately truncated by static non-negativity boundaries. To overcome the rank fragmentation inherent in independent multi-head processing \citep{knockinghead, hybridLA}, we extend MirrorLA to the global scale via a unified \textit{cross-head reflection}. This global operator couples feature spaces across heads through covariance mixing, allowing dormant subspaces to inherit discriminative structure from informative heads and improving overall expressivity.

We evaluate our method on five vision benchmarks across nine datasets, covering image classification, detection and segmentation, semantic segmentation, diffusion models, and ultra-long sequence super-resolution. 
Our model consistently outperforms baselines, with gains of up to 4.4\% on classification, 4.7\% on detection and segmentation, and a 6.6 mIoU improvement on semantic segmentation under lower FLOPs.
Notably, for ultra-long sequence super-resolution, it achieves higher accuracy while reducing memory by up to 81.4\% and inference time by up to 78\%, and also yields lower FID in diffusion tasks.

% Transformer based model \citep{transformer, vit} 已经展现出亮眼的性能in both vision and natural language tasks. Its core component, self-attention, 能够通过内积相似度与softmax归一化有效捕获全文信息，但是导致了平方复杂度，对长序列文本或高分辨率图片/视频的场景下带来了巨大的挑战。为了解决这一问题，linear attention\citep{linearattn}用linearly separable kernel代替了softmax中的exp()。这样重新定义相似度的自注意力机制可以利用矩阵乘法结合律，实现用线性复杂度计算自注意力机制。

% 近期的工作希望能够避免核函数带来的负值缺失问题，polaformer采用了sign分解的方式，让向量做内积计算时可以按照符号，让负值也参与计算；hedgehog通过衔接mlp再加上exp的方式让模型学习到softmax attention的weight，nalaformer则按照向量维度的余弦相似度代替传统的内积相似度。但这些方法都引入了过多的计算量，增加了feature dim的维度，这对于$O(Nd^2)$的线性注意力机制是不友好的。

% 尽管linear attention凭借高速度和低显存占用等优势被广泛改进，在对比他们的softmax-counterparts时，仍然underperform。作为其核心要素，kernel的设计成为研究热点，, which approximates attention through inner products of transformed queries and keys。早期工作主要focus在相似度的非负约束上， which 是一个必要条件保证归一化因子不为0，否则会导致训练过程中的不稳定和梯度爆炸。To this end, various non-negative activation functions have been employed。

% 决定linear表达能力的核心要素是kernel function \phi的设计，which 定义了相似度\phi(q)\phi(k)^\top。但受限于相似度的非负与线性可分两个约束，导致不得不把\phi设计成逐元素非负映射，因此现有linear attention任务表现往往是不如softmax-based counterparts的。已有的工作采用ReLU、1+ELU、exp等逐元素映射方式把所有的计算都压缩在$\mathbb{R}^+ \cup \{0\}$上进行计算。
\section{Related Work}
\textbf{Efficient Sequential Vision Models.}
Treating images as token sequences has become a standard approach in modern vision architectures, inspired by the success of Transformer models in natural language processing \citep{transformer}. 
Vision Transformer (ViT) \citep{vit} demonstrates that patchified images can be processed by self-attention to capture long-range dependencies, but its reliance on softmax attention incurs quadratic computational and memory costs, which quickly become prohibitive for high-resolution vision tasks.
To mitigate this issue, a line of research focuses on improving the efficiency of Transformer-based vision models while largely preserving the attention formulation.
DeiT \citep{deit} leverages distillation strategies to enhance data efficiency, and Swin Transformer \citep{swin,swinv2} restricts attention to shifted local windows, enabling scalable training while maintaining cross-window information flow.
These methods primarily reduce constant factors or constrain attention scope, rather than fundamentally resolving the quadratic complexity of softmax attention.
More recent efforts explore alternatives beyond standard attention mechanisms to further improve scalability.
State space models (SSMs) have been adapted to vision by sequentially processing patch tokens, with VMamba \citep{vmamba,localvmamba} adopting raster-scan ordering to model spatial dependencies under linear-time complexity.
In a different vein, VHeat \citep{vheat} formulates visual representation learning as a diffusion-like process and models patch interactions through heat conduction, achieving sub-quadratic complexity via discrete cosine transform (DCT)-based operations.

\textbf{Linear Attention.}
Linear attention reduces the quadratic complexity of softmax attention by replacing the exponential similarity with kernelized feature maps, enabling associative reordering and linear-time computation \citep{linearattn}. 
Early approaches focus on approximating softmax using separable kernels, with representative designs based on ReLU \citep{flatten,efficientvit}, $\exp$ \citep{lognormalattention}.
Performer \citep{performer} utilizes Random Fourier Features (RFF) to provide an unbiased estimate of the softmax kernel, while CosFormer \citep{cosformer} employs a linearithmic complexity approach with a ReLU-based kernel and cosine re-weighting to enforce locality.
Linformer \citep{linformer} approximates the self-attention matrix via low-rank decomposition, projecting the key and value sequences into a lower-dimensional subspace to achieve linear complexity.
Beyond kernel selection, recent methods further introduce data-dependent modulation or gating mechanisms to stabilize large-scale training \citep{lightningattn,minimax}.
While effective, such approaches still rely on fixed geometric parameterizations and do not explicitly address the information loss induced by kernelization.
From a theoretical perspective, PolaFormer \citep{polaformer} mitigates information loss from non-negative feature maps via sign-aware inner product decomposition and using power function to lower the attention entropy.
SAGA \citep{saga} and RALA \citep{rala} exploit the property of the Hadamard product to enhance the intermediate matrices, thereby alleviating the low-rank bottleneck in linear attention.
In-Line \citep{inline} shows that softmax attention is injective under mild conditions, which is generally not in linear attention, and MALA \citep{mala} and NaLaFormer \citep{nalaformer} also link this gap to information loss induced by non-negative feature maps.
Consequently, while existing studies strive to preserve non-negativity, they largely overlook the fact that conventional activations as \textit{passive truncation} mechanisms, incurring severe information loss and limiting the representational capacity of linear attention.
% \input{sections/2-preliminaries}
% \subsection{Preliminaries}
\section{Method}
\noindent\textbf{The Information Loss of Linear Attention}. Given an input sequence of length $N$ and dimension $d$, Query ($\mathbf{Q}$), Key ($\mathbf{K}$), and Value ($\mathbf{V}$) matrices are in $\mathbb{R}^{N \times d}$. Standard attention computes pairwise similarities via:
\begin{align}
    &\mathbf{o}_t =
    \sum_{i=1}^{N} \frac{ \operatorname{exp} (\mathbf{q}_t \mathbf{k}_i^{\top} /\sqrt{d})}
    {\sum_{j=1}^{N} \operatorname{exp} (\mathbf{q}_t\mathbf{k}_j^{\top} /\sqrt{d})}\mathbf{v}_i,
    \label{softmaxnorm}
\end{align}
which results in $\mathcal{O}(N^2d)$ complexity. To alleviate this, Linear Attention (LA) re-formulates the attention mechanism by replacing the softmax-normalized dot product with a generalized kernel function $\operatorname{sim}(\mathbf{q}, \mathbf{k}) = \phi_{r}(\mathbf{q})\phi_{r}(\mathbf{k})^{\top}$. By leveraging the associative property of matrix multiplication, the complexity can be reduced to  $\mathcal{O}(Nd^2)$ as
\begin{align}
    \mathbf{o}_t=& 
    \frac{ \sum_{i=1}^{N}\phi_{r}(\mathbf{q}_t)\phi_{r}(\mathbf{k}_i)^{\top}\mathbf{v}_i}
    {\sum_{j=1}^{N} \phi_{r}(\mathbf{q}_t)\phi_{r}(\mathbf{k}_j)^{\top}}=\frac{\phi_{r}(\mathbf{q}_t)\sum_{i=1}^{N}\phi_{r}(\mathbf{k}_i)^{\top}\mathbf{v}_i}
    {\phi_{r}(\mathbf{q}_t)\sum_{j=1}^{N} \phi_{r}(\mathbf{k}_j)^{\top}}.
    \label{linearattn}
\end{align}

However, a critical constraint in this formulation is that the choice of the kernel feature map $\phi_{r}(\cdot)$ must ensure \textit{non-negativity} $\operatorname{sim}(\mathbf{q}, \mathbf{k})\geq0$. Otherwise, the denominator $\phi(\mathbf{q})\sum \phi(\mathbf{k_j})^\top$ will approach zero, resulting in numerical instability and gradient explosion during training. To satisfy this for arbitrary data distributions of $\mathbf{Q}$ and $\mathbf{K}$, previous methods typically employ \textit{axis-aligned} mappings such as $\operatorname{ReLU}(\cdot)$ or $\exp(\cdot)$ to constrain the feature maps to the non-negative orthant. 
Despite their simplicity, such projections act as a hard thresholding operator and indiscriminately \textit{suppress} features in the negative domain (as depicted in Fig.\ref{fig:mainfig}), leading to irreversible information loss that hampers the model's expressive capacity.

\begin{figure*}
    \centering
    \includegraphics[width=0.95\linewidth]{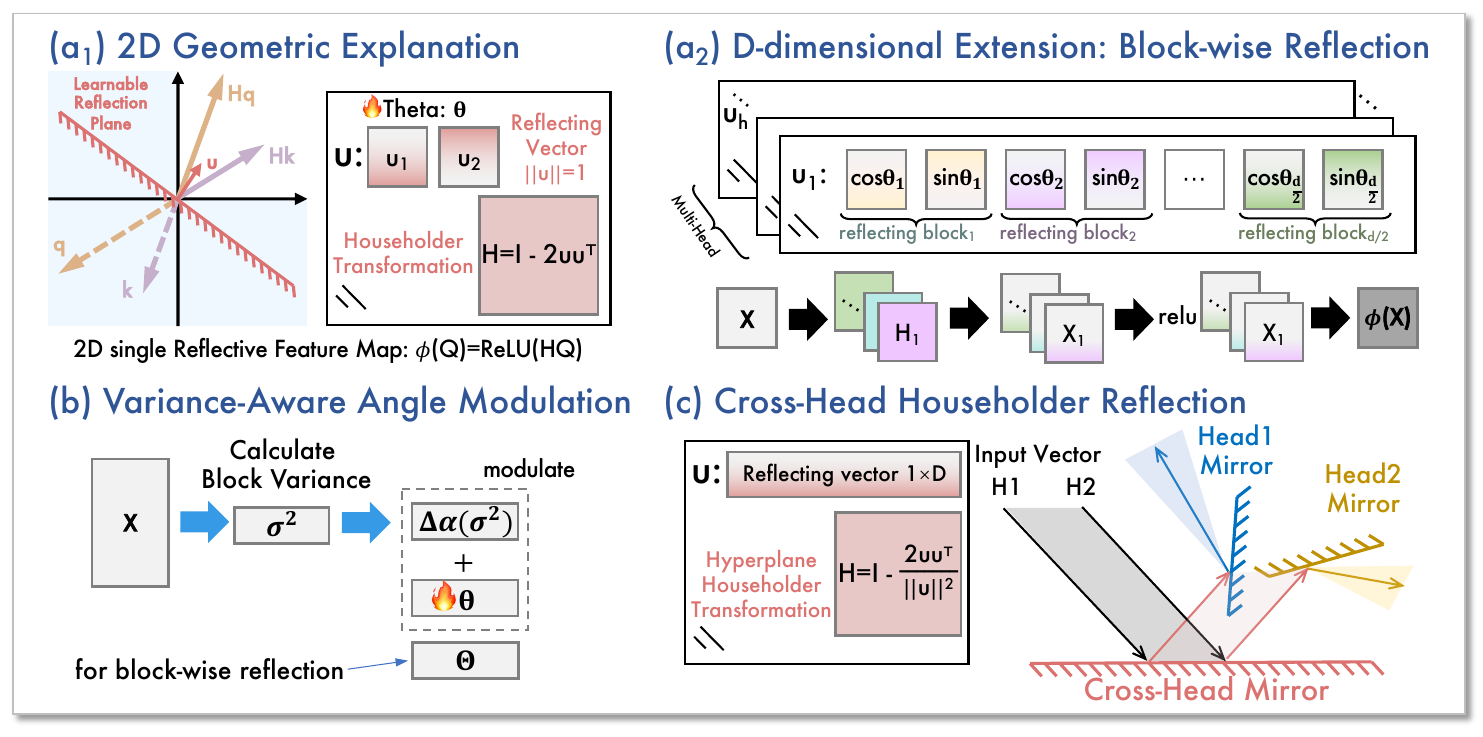}
    \caption{\textbf{Overview of the MirrorLA framework}. (a1) In 2D, a learnable Householder reflection $\mathbf{H}$ reorients features $\mathbf{Hq}$ $\mathbf{Hk}$ before applying an axis-aligned non-negativity map, converting passive truncation into active alignment while preserving inner products by isometry. (a2) Extension to $D$-dim via block-wise and reflection with low overhead (b) Adaptively adjusting reflection angles based on block variance $\sigma^2$ to enhance feature diversity. (c) A global transformation before head-wise decomposition to encourage inter-head communication.}
    \label{fig:mainfig}
\end{figure*}

\textbf{Overview of MirrorLA.} To resolve the conflict between the non-negativity constraint and information preservation, we propose MirrorLA. Drawing on deep geometric isometries \citep{mhammedi2017efficient}, MirrorLA introduces a learnable reflection-based mechanism. Rather than passively \textit{truncating} negative features, we optimally \textit{reorient} the feature geometry into the non-negative orthant via Householder transformations prior to activation so that the QK interactions can be steadily preserved. MirrorLA comprises three components: (1) \textit{Block-wise Householder isometries} for local feature preservation (Sec.  \ref{sec:single_head}); (2) \textit{Variance-aware modulation} to counteract oversmoothing in long-context modeling (Sec.  \ref{sec:single_head}); (3) \textit{Cross-head reflection} to enable inter-head communications and global covariance mixing (Sec. \ref{sec:cross_head}).

\subsection{Block-wise Householder Isometries}\label{sec:single_head}
\noindent\textbf{Householder Transformation}.
To preserve feature magnitude while satisfying non-negativity, we employ Householder reflections. The Householder matrix associated with a unit mirror vector $\mathbf{u}\in\mathbb{R}^d$ defines a reflection across the hyperplane orthogonal to $\mathbf{u}$:
\begin{align}
\mathbf{H_u}=\mathbf{I}-2\frac{\mathbf{uu^\top}}{||\mathbf{u}||_2^2}.
\end{align}
This transformation is orthogonal ($\mathbf{H}_{\mathbf{u}}^\top\mathbf{H}_{\mathbf{u}}=\mathbf{I}$), ensuring it is an isometry that preserves Euclidean norms. Accordingly, the reflection of an input vector $\mathbf{x}$ with respect to this hyperplane is obtained as $\mathbf{H_u x}$. We first show a simplified 2D case to build a geometric interpretation:

% Then we describe how this idea can be extended to high-dimensional query and key vectors via a block-wise formulation.

\noindent\textbf{2D Case: Intuition of Mirrored Feature Map}.  
Let the mirror vector $\mathbf{u} = [u_1, u_2]^\top \in \mathbb{R}^2$ be a unit mirror vector and the input be $\mathbf{x} \in \mathbb{R}^2$.  
The corresponding Householder reflection is:
\begin{align}
    \tilde{\mathbf{x}} = \mathbf{H_u x}, \quad 
    \mathbf{H_u} = \mathbf{I} - 2\mathbf{u}\mathbf{u}^\top,
\end{align}
where $\|\mathbf{u}\|_2=1$. Parameterizing $\mathbf{u}$ by an angle $\theta$,
\begin{align}
    u_1 = \cos\theta, \quad u_2 = \sin\theta.
\end{align}
we obtain the closed form 2D Householder formula as
\begin{align}
    \mathbf{H_u} 
    &= 
    \begin{bmatrix}
        1 - 2 u_1^2 & -2 u_1 u_2 \\
        -2 u_1 u_2 & 1 - 2 u_2^2
    \end{bmatrix} =
    \begin{bmatrix}
        \cos 2\theta & \sin 2\theta \\
        \sin 2\theta & -\cos 2\theta
    \end{bmatrix}.
\end{align}
Thus, in two dimensions, the Householder transformation can be interpreted as a reflection across a line with angle $\theta$, providing an intuitive view of how the input is mirrored before further non-negative mapping such as $\operatorname{ReLU}$, which is defined as,
\begin{align}
    \phi_{r}(\mathbf{x})=\operatorname{ReLU}(\mathbf{H_u x})
\end{align}
As shown in Fig.~\ref{fig:mainfig} (left), a learnable unit vector $\mathbf{u}$ parameterized by an angle $\theta$ allows the model to optimally rotate the mirror geometry relative to the ReLU activation boundary. Instead of passive truncation, the elements are reflected into the active region, preserving structural information that would otherwise be discarded. Consequently, the reflected features retain richer information and display more balanced statistics, which enhances the expressiveness of the feature map and yields more \textit{stable} linear attention estimation.

\noindent\textbf{Extension to High-Dimension: Block-wise Reflection}. 
To scale this to arbitrary dimensions $D$ efficiently, we partition the query or key feature vector $\mathbf{x}\in\mathbb{R}^{D}$ into $M=D/2$ disjoint two-dimensional blocks:
\begin{align}
\mathbf{x} = \big[\mathbf{x}_1,\mathbf{x}_2,\ldots,\mathbf{x}_M\big],\quad
\mathbf{x}_m\in\mathbb{R}^2.
\end{align}
For each block $m$, we associate a learnable mirror vector $\mathbf{u}_m=\begin{bmatrix}
\cos\theta_m\
\sin\theta_m
\end{bmatrix}$ and reflect each block independently:
\begin{align}
\mathbf{H}_m = \mathbf{I}_2 - 2\mathbf{u}_m (\mathbf{u}_m)^\top, ~\tilde{\mathbf{x}}_m = \mathbf{H}_m \mathbf{x}_m,
\end{align}
Stacking the reflected blocks gives the full transformed vector, which is followed by $\operatorname{ReLU}$ to enforce non-negativity:
\begin{align}
\phi_{r}(\mathbf{x})=\operatorname{ReLU}\left(\tilde{\mathbf{x}}\right), ~\tilde{\mathbf{x}}=
\big[\tilde{\mathbf{x}}_1,\ldots,\tilde{\mathbf{x}}_M\big]\in\mathbb{R}^D.
\end{align}
Given a query $\mathbf{q}$ and key $\mathbf{k}$, their linear attention kernel is computed in a blockwise \textit{additive} manner. For each head $h$, we evaluate the similarity by summing contributions from all 2D blocks:
\begin{align}
\langle \phi_{r}(\mathbf{q}),\phi_{r}(\mathbf{k}) \rangle
= \sum_{m=1}^{M}
\left\langle
\phi_{r}(\mathbf{q}_m),
\phi_{r}(\mathbf{k}_m)
\right\rangle ,
\end{align}
which is equivalent to the standard inner product over $\mathbb{R}^D$ but explicitly decomposed into 2D reflective components.
This block-wise design imposes a local mixing bias, reorienting features within small coordinate groups rather than globally entangling all dimensions. This yields stable geometry adaptation with $\mathcal{O}(d)$ complexity and bounded operator norm inherited from orthogonal reflection. 

\subsection{Rotating the Mirror: Variance-Aware Angle Modulation}\label{sec:variance}
While block-wise reflections optimize static geometry, long-context modeling suffers from \textit{feature oversmoothing}, where token embeddings concentrate in a narrow subspace. If the reflection angle $\theta_m$ is static, hese aligned tokens may all map to the same sign pattern, causing the subsequent ReLU to zero-out entire sequence segments identically.

To mitigate this, we introduce a \emph{variance-aware angle modulation} mechanism that \textit{adaptively} perturbs the reflection angle based on the block variance $\sigma_{m}^2$ across the sequence length $L$, encouraging activation-pattern diversity in low-variance subspaces. 
Let $\mathbf{x}_{t,m}\in\mathbb{R}^2$ denote the feature vector of block $m$ at token $t$, and let $\bar{\mathbf{x}}_m=\frac{1}{L}\sum_{t=1}^{L}\mathbf{x}_{t,m}$ be the emprical mean. We define the block-level variance as
\begin{equation}
\sigma_m^2
=
\frac{1}{2L}
\sum_{t=1}^{L}
\|
\mathbf{x}_{t,m}
-
\bar{\mathbf{x}}_{m}
\|^2.
\end{equation}
Intuitively, a small $\sigma_m^2$ indicates that token representations within this block carry weak token-specific information, and thus benefit from a stronger reorientation relative to the activation boundary. We accordingly modulate the reflection angle by:
\begin{align}
&\Theta_m = \theta_m + \Delta\alpha(\sigma^2), \label{eq:mod}\\&\Delta\alpha(\sigma^2)=\operatorname{sigmoid}\left(\frac{\lambda}{\sigma_m^2 + \varepsilon}\right)\cdot\alpha_{\operatorname{max}},
\end{align}
where $\theta_m$ denotes the learnable base reflection angle for block $m$, $\alpha_{\operatorname{max}}\in \{\pi/8, \pi/4, \pi/2\}$ is a scaling hyperparameter controlling the maximum perturbation, $\lambda$ controls the variant sensitivity, and $\varepsilon$ is a small constant for numerical stability. We then replace $\theta_m$ with the variance-aware angle $\Theta_m$ in the Householder reflection, while keeping the per-block computation linear in $D$.
\paragraph{Theoretical Motivation.}
The following theorem formalizes how this modulation promotes activation diversity.

\begin{theorem}{Collapse-aware Activation Diversification}{variance_shift}
Let $\mathcal{X}=\{\mathbf{x}_{t,m}\}_{t=1}^{L}\subset\mathbb{R}^{2}$ be the token features of a 2D block with empirical variance $\sigma^{2}$ computed over the token dimension. Consider the modulated mirror map
\begin{equation}
\hspace{-1ex}\phi(x;\Theta) \;=\; \mathrm{ReLU}\big(\mathbf{H}_m(\Theta_m)\mathbf{x}_{t,m}\big),
\end{equation}
where $\Delta\alpha(\sigma^2)=\mathrm{sigmoid}(1/(\sigma^2+\varepsilon))\cdot\alpha$ and $\mathbf{H}(\Theta_m)$ denotes a 2D Householder reflection parameterized by angle $\Theta_m$ (Eq. \ref{eq:mod}).
Define the binary activation mask:
\begin{equation}
s(\mathbf{x}_{t,m};\Theta_m)\;=\;\mathbb{I}\!\left[\mathbf{H}(\Theta_m)\mathbf{x}_{t,m}>\mathbf{0}\right]\in\{0,1\}^2.
\end{equation}
Then, in the low-variance regime $\sigma^2\to 0$, the modulation saturates $\Theta_m\to\theta+\alpha_{\operatorname{max}}$ and induces a non-degenerate reorientation of the reflection hyperplane. In particular, for perturbative tokens $\mathbf{x}_{t,m}=\bar{\mathbf{x}}_{m}+\delta_t$, the expected mask disagreement satisfies
\begin{align}\nonumber
\hspace{-1ex}\mathbb{E}_{t\neq t'}&\!\left[\|s(\mathbf{x}_{t,m};\Theta)-s(\mathbf{x}_{t',m};\Theta)\|_0\right]
\\&\ge
c(\bar{\mathbf{x}}_{m},\Theta)\cdot
\mathbb{E}_{t\neq t'}\!\left[\|\mathbf{x}_{t,m}-\mathbf{x}_{t',m}\|_2\right],
\end{align}
for some $c(\bar{\mathbf{x}},\Theta_m)>0$ whenever $\mathbf{H}(\Theta_m)\bar{\mathbf{x}}$ lies within a bounded margin of the ReLU boundary.
\end{theorem}
% Theorem~\ref{thm:variance_shift}.

\noindent\textit{Proof sketch.}
As $\sigma^2\to 0$, we have $1/(\sigma^2+\varepsilon)\to\infty$, hence $\Delta\alpha(\sigma^2)$ saturates near $\alpha_{\operatorname{max}}$, yielding the maximal angular shift. Mask disagreement arises when tokens fall on different sides of the coordinate hyperplanes after reflection, which is controlled by the distance (margin) of $\mathbf{H}(\Theta_m)\bar{\mathbf{x}}$ to the ReLU boundary. By rotating the reflection hyperplane in the low-variance regime, the proposed modulation prevents an entire oversmoothed token cluster from sharing an identical activation pattern, thereby mitigating collapse under long-context accumulation.\hfill$\square$

\subsection{Cross-Head Householder Reflection}
\label{sec:cross_head}
While multi-head factorization improves expressivity, purely head-wise processing implicitly limits interactions across heads and can lead to fragmented representations, especially under non-negative feature maps. Recent studies on \emph{inter-head communication} \citep{knockinghead,hybridLA} suggest that lightweight transformations along the head dimension $H$ can substantially enhance attention capacity by allowing heads to exchange complementary information. Motivated by this perspective, we introduce a global cross-head Householder reflection that couples the full feature space \emph{before} head-wise decomposition, enabling global correlation mixing with stable conditioning.

Formally, let $\mathbf{X} \in \mathbb{R}^{L \times HD}$ denote the concatenated multi-head features, where $H$ is the number of heads and $D$ is the per-head dimension. We apply a global Householder reflection
\begin{equation}
\phi_c(\mathbf{X}) =\mathbf{X} \mathbf{H}_c,
\quad
\mathbf{H}_c = \mathbf{I}_{HD} - 2 \frac{\mathbf{u}_c \mathbf{u}_c^\top}{\|\mathbf{u}_c\|_2^2},
\end{equation}
where $\mathbf{u}_c \in \mathbb{R}^{HD}$. Importantly, the transformation preserves the Euclidean geometry while introducing controlled cross-head mixing.

\paragraph{Covariance Mixing.}
Let $\Sigma=\mathrm{Cov}(\mathbf{X})\in\mathbb{R}^{HD\times HD}$ denote the covariance of the concatenated head features at a token. Since $\mathbf{H}_c$ is orthogonal, the reflected covariance satisfies,
\begin{equation}
\Sigma' \;=\; \mathrm{Cov}(x \mathbf{H}_c) \;=\; \mathbf{H}_c^\top \Sigma \mathbf{H}_c,
\end{equation}
which preserves the covariance spectrum but generally redistributes correlation mass across head blocks. As a result, $\phi_c(\cdot)$ mitigates the block-diagonal bias induced by independent head processing, allowing low-variance (``dead'') heads to recover discriminative directions from high-variance heads prior to the subsequent head-wise non-negative mapping. After the cross-head reflection, $\phi_c(\mathbf{X})$ is reshaped back into $H$ heads and processed with the variance-aware block-wise Householder reflections from \S\ref{sec:variance}, followed by $\mathrm{ReLU}$ per head. The complete MirrorLA pipeline is summarized in Algorithm~\ref{algo:code}.

\begin{algorithm}[h]
\caption{Reflecting Feature Map}
\label{algo:code}
\definecolor{codeblue}{rgb}{0.25,0.5,0.5}
\lstset{
    backgroundcolor=\color{white},
    basicstyle=\fontsize{8.2pt}{8.6pt}\ttfamily\selectfont,
    columns=fullflexible,
    breaklines=true,
    captionpos=b,
    commentstyle=\fontsize{7.6pt}{7.6pt}\color{codeblue},
    keywordstyle=\fontsize{8.2pt}{8.2pt},
}
\begin{lstlisting}[language=python]
# x: input features [B, N, D]
# H: number of heads, d = D / H

# ----- Global reflection -----
u_c = normalize(randn(D))
x = x - 2 * u_c * <x, u_c>
# ----- Head-wise reshape -----
x = reshape(x, [B, H, N, d])
# ----- Block partition -----
x_blk = reshape(x, [B, H, N, d//2, 2])
# ----- Variance-aware angle modulation -----
theta = (randn(H, d//2)
var = Var(x_blk over N and 2D)
delta = sigmoid(lambda / (var + epsilon)) * alpha_max
theta = theta + delta
# ----- Block-wise reflection -----
u = [cos (theta), sin (theta)]
x_blk = x_blk - 2 * u * <x_blk, u>
# ----- Restore shape -----
x = reshape(x_blk, [B, H, N, D])
return ReLU(x)
\end{lstlisting}
\end{algorithm}

\section{Experiments}
In this section, we conduct extensive experiments to evaluate the effectiveness of MirrorLA including image classification, object detection and instance segmentation, diffusion transformer and high-resolution dense prediction tasks such as semantic segmentation and super-resolution.

\subsection{Image Classification on ImageNet-1K}
\noindent\textbf{Settings.}
On ImageNet-1K \citep{imagenet}, we evaluate MirrorLA by training from scratch and reporting Top-1 accuracy.
To provide a comprehensive evaluation, we compare our method against various State-of-the-Art (SOTA) efficient vision models. 
For a fair and structured analysis, these baselines are categorized into several tiers based on parameter count, allowing for direct performance comparisons within similar model scales.
% \input{tables/classification}
% Two-column layout inside a single-column document:
% left: ImageNet-1K classification; right: COCO + Segmentation
\begin{table}[!t]
    \centering

    \begin{minipage}[t]{0.49\linewidth}
        \centering
        \captionof{table}{\textbf{Comparison of the ImageNet-1K classification with the SOTA efficient models.} The ``\textsc{Params}'' column represents the number of parameters, ``\textsc{Acc}'' for Top-1 accuracy.}
        \label{tab:imagenet}
        \setlength{\tabcolsep}{2pt}
        \resizebox{\linewidth}{!}{
            \begin{tabular}{l|cc|c}
                \toprule[2pt]
                \textsc{Model} & \textsc{Params} & \textsc{FLOPs} & \textsc{Acc.}\\
                \midrule[1pt]
                Conv2Former-N \citep{conv2f}   &15M    &2.2G   &81.5 \\
                RMT-T \citep{rmt}  &14M & 2.5G & 82.4 \\
                Agent-PVT-T \citep{agentattn} &12M    &2.0G   &78.4 \\
                % EfficientViT-B2 \citep{efficientvit}&24M&1.6G  &82.1\\
                RMT-T \citep{rmt}&14M&2.5G  &82.4 \\
                RAVLT-T \citep{rala}     &15M    &2.6G   &82.6 \\
                % NaLaFormer-T \citep{nalaformer}  & 16M &2.7G  &82.6 \\
                \rowcolor{blue!10}
                MirrorLA-T & 15M & 2.6G &\textbf{82.8}\\

                \midrule[1pt]
                Conv2Former-T \citep{conv2f}     &27M    &4.4G   &83.2 \\
                MambaOut-T \citep{mambaout}     &27M    &4.5G   &82.7 \\
                % MogaNet-S \citep{moganet}     &25M    &5.0G   & 83.4 \\
                VMamba-T \citep{vmamba}     &30M    &4.9G   &82.6 \\
                % InternImage-T \citep{internimage}   &30M    &5.0G   &83.5 \\
                % EfficientViT-B3 \citep{efficientvit}&49M&4.0G  &82.1\\
                % FL-Swin-S \citep{flatten} &51M    &8.7G   &83.5 \\
                Agent-Swin-T \citep{agentattn}    &29M    &4.5G   &82.6 \\
                Vim-S \citep{vim}  &26M    &3.7G   &80.6 \\
                VMamba-T \citep{vmamba}   &30M    &4.9G   &82.6 \\
                LocalVMamba-T \citep{localvmamba} &26M    &5.7G   &82.7 \\
                % CSwin-T \citep{cswin}      &23M    &4.3G   &82.7 \\
                % SG-Former-S \citep{sgformer}    &23M    &4.8G   &83.2 \\
                MOAT-0 \citep{moat}       &28M    &5.7G   &83.3 \\
                % RMT-S \citep{rmt}&27M & 4.5G & 84.1 \\
                Pola-Swin-T \citep{polaformer}   &29M    &4.5G   &82.6 \\
                ViG-H-T \citep{vig}  &29M    &4.5G   &82.8 \\
                MILA-T \citep{mlla}  &25M    &4.2G   &83.5 \\
                \rowcolor{blue!10}
                MirrorLA-S & 26M & 4.9G  &\textbf{84.2}\\

                \midrule[1pt]
                MambaOut-S \citep{mambaout}   &49M    &9.0G   &84.1 \\
                MogaNet-B \citep{moganet}   &44M    &10G   &84.3 \\
                VMamba-S \citep{vmamba}  &50M    &8.7G   &83.6 \\
                StructViT-B-8-1 \citep{structvit}  &52M    &12G   &84.3 \\
                SOFT-L \citep{SOFT}  &64M    &11G   &83.1 \\
                FL-Swin-S \citep{flatten}  &51M    &8.7G   &83.5\\
                Agent-Swin-S \citep{agentattn}   &50M    &8.7G   &83.7 \\
                Pola-Swin-S \citep{polaformer} &50M    &8.7G   &83.6 \\
                ViG-H-S \citep{vig}   &50M    &8.8G   &83.8 \\
                RMT-B \citep{rmt}&54M & 9.7G & 85.0 \\
                MILA-S \citep{mlla}    &43M    &7.3G   &84.4 \\
                RAVLT-B \citep{rala}    &48M    &10.6G   & 85.1\\
                % NalaFormer-B \citep{nalformer}  & 52M  & 12G  & 85.2\\
                \rowcolor{blue!10}
                MirrorLA-B   & 48M  & 10.6G & \textbf{85.3}\\

                \midrule[1pt]
                % InterImage-B \citep{internimage}  &97M    &16G   &84.9 \\
                MambaOut-B \citep{mambaout}    &85M    &15.8G   &84.2 \\
                VMamba-B \citep{vmamba}   &89M    &15.4G   &83.9 \\
                % SG-Former-B \citep{sgformer}   &78M    &16G  &84.7 \\
                FL-Swin-B \citep{flatten}   &89M    &15.4G   &83.8 \\
                Agent-Swin-B \citep{agentattn} &88M    &15.4G   &84.0 \\
                Pola-Swin-B \citep{polaformer}    &88M    &15.4G   &83.8 \\
                SMT-L \citep{smt}       &81M&   17.7G     &84.6 \\
                VRWKV-B \citep{vrwkv}    &94M    &18.2G   &82.0 \\
                ViG-H-B\citep{vig}    &89M    &15.5G   &84.2 \\
                MaxViT-B\citep{maxvit}    &120M    &23.4G   &84.9 \\
                GC-ViT-B\citep{gcvit}   &90M    &14.8G   &85.0 \\
                RMT-L \citep{rmt}  &95M & 18.2G & 85.5 \\
                InLine-Swin-B \citep{inline}  &88M    &15.4G   &82.0 \\
                MILA-B \citep{mlla} &96M    &16.2G   &85.3 \\
                % ViG-H-B \citep{vig} &89M    &16G   &84.2 \\
                RAVLT-L \citep{rala}    &95M    &17.3G   &85.5 \\
                \rowcolor{blue!10}
                MirrorLA-L  & 95M  & 17.3G & \textbf{85.7}\\
                \bottomrule[2pt]
            \end{tabular}
        }
    \end{minipage}
    \hfill
    \begin{minipage}[t]{0.49\linewidth}
        \centering
        \captionof{table}{\textbf{Object detection and instance segmentation results on the COCO dataset.} The ``P'' column represents the number of parameters, and ``F'' for FLOPs.}
        \setlength{\tabcolsep}{3pt}
        \label{tab:COCO}
        \resizebox{\linewidth}{!}{
            \begin{tabular}{l|cc|cccccc}
                \toprule[2pt]
                \multirow{2}{*}{\textsc{Model}}& \textsc{P}&\textsc{F}&\multicolumn{6}{c}{\textsc{RetinaNet 1$\times$}}\\
                   &(M)&(G)& \textsc{AP}$^{b}$ & \textsc{AP}$^{b}_{50}$ & \textsc{AP}$^{b}_{75}$ & \textsc{AP}$^{b}_{S}$ & \textsc{AP}$^{b}_{M}$ & \textsc{AP}$^{b}_{L}$\\
                \midrule

                {PVT-T\citep{pvtv2}}&23&221& 39.4 & 59.8 & 42.0 & 25.5 & 42.0 & 52.1\\

                {MPViT-XS\citep{mpvit}}& 20&211&43.8& 65.0 &47.1& 28.1& 47.6 &56.5\\

                {SOFT++  T\citep{SOFT}} & 23&200& 41.9 & 62.7 & 44.7 & 27.8 & 45.4 & 55.6\\
                {RMT-T\citep{rmt}} &23 &199& 45.1 &66.2 &48.1 &28.8& 48.9 &61.1 \\
                {RAVLT-T\citep{rala}}& 24& 201 &44.1& 64.3 &47.4& 26.4&48.1&59.3  \\
                \rowcolor{blue!10}
                MirrorLA-T &24&200 & \textbf{46.1} & \textbf{67.4} & \textbf{49.4} & \textbf{28.8} & \textbf{50.4} & \textbf{61.7}\\
                \midrule
                {Swin-T\citep{swin}}&38 &248& 41.7& 63.1& 44.3 &27.0 &45.3 &54.7\\
                {MPViT-S\citep{mpvit}}&32&248&45.7& 57.3& 48.8& 28.7 &49.7& 59.2\\
                {PVTv2-B2\citep{pvtv2}}&35&290&44.6& 65.6 &47.6& 27.4& 48.8 &58.6\\
                {CMT-S\citep{cmt}}&44&231&44.3 & 65.5 & 47.5 & 27.1 & 48.3 & 59.1\\

                {RAVLT-S\citep{rala}}& 44& 262 &46.7& 67.2 &50.4& 31.4&51.1&62.3  \\\rowcolor{blue!10}
                {MirrorLA-S}&34&243& \textbf{48.5} & \textbf{69.8} & \textbf{52.2} & \textbf{32.2} & \textbf{52.9} & \textbf{63.7}  \\

                \midrule[1pt]

                \multirow{2}{*}{\textsc{Model}}& \textsc{P}&\textsc{F}& \multicolumn{6}{c}{\textsc{Mask R-CNN 1$\times$}}\\
                   &(M)&(G)& \textsc{AP}$^{b}$ & \textsc{AP}$^{b}_{50}$ & \textsc{AP}$^{b}_{75}$ & \textsc{AP}$^{m}$ & \textsc{AP}$^{m}_{50}$ & \textsc{AP}$^{m}_{75}$ \\
                \midrule

                {PVTv2-b1\citep{pvtv2}}&33&243&41.8& 54.3 &45.9 &38.8& 61.2& 41.6 \\

                {MPViT-XS\citep{mpvit}}&30&231&47.3 &69.1 &51.9 &42.7& 66.2& 46.0\\

                {FL-PVT-T\citep{flatten}}&32&244&  38.2 & 61.6 & 41.9 & 37.0 & 57.6 & 39.0\\
                {RAVLT-T\citep{rala}}& 33& 219 &47.2& 69.1 &51.7& 42.7&66.0&46.0  \\
                % {MAVLT-T\citep{mala}}& 33& 226 &47.5& 69.0 &52.3& 43.0&66.3&46.3  \\
                \rowcolor{blue!10}
                {MirrorLA-T}&33&219& \textbf{47.7} & \textbf{69.6} & \textbf{52.6} & \textbf{43.0} & \textbf{66.6} & \textbf{46.2} \\
                \midrule

                {InternImage-T\citep{internimage}}& 49&270&47.2& 69.0 &52.1 &42.5 &66.1& 45.8  \\

                {MambaOut-T\citep{internimage}}&43&262&45.1& 67.3& 49.6& 41.0 &64.1& 44.1  \\

                {MPViT-S\citep{mpvit}}&43&268& 46.4& 68.6& 51.2 &42.4 &65.6 &45.7 \\

                {CMT-S\citep{cmt}}& 45&249&44.6 &66.8 &48.9 &40.7& 63.9 &43.4 \\

                {MILA-T\citep{mlla}}&44&262& 46.8 & 69.5 & 51.5& 42.1& 66.4 &45.0\\

                \rowcolor{blue!10}
                {MirrorLA-S}&44&262& \textbf{49.9} & \textbf{71.6} & \textbf{54.7} & \textbf{44.5} & \textbf{68.4} & \textbf{48.1}\\

                \midrule[1pt]

                \multirow{2}{*}{\textsc{Model}}& \textsc{P}&\textsc{F} &  \multicolumn{6}{c}{\textsc{Mask R-CNN 3$\times$}}\\
                &(M)&(G)& \textsc{AP}$^{b}$ & \textsc{AP}$^{b}_{50}$ & \textsc{AP}$^{b}_{75}$ & \textsc{AP}$^{m}$ & \textsc{AP}$^{m}_{50}$ & \textsc{AP}$^{m}_{75}$  \\

                \midrule

                {XCiT-T12/8\citep{xcit}}&26&266& 44.5 &66.4 &48.8 &40.4& 63.5 &43.3  \\
                {RMT-T\citep{rmt}} &33& 218& 47.1 &68.8 &51.7 &42.6 &65.8 &45.9 \\
                {MPViT-T\citep{mpvit}}&28&216& 44.8 & 66.9 & 49.2 & 41.0 & 64.2 & 44.1\\
                \rowcolor{blue!10}
                {MirrorLA-T}&33&129& \textbf{49.4} & \textbf{70.7} & \textbf{54.3} & \textbf{44.1} & \textbf{67.7} & \textbf{47.7}  \\
                \midrule

                {InternImage-T\citep{internimage}} &49&270& 49.1 &70.4& 54.1& 43.7 &67.3& 47.3  \\

                {RMT-S\citep{rmt}} &46&262&50.7& 71.9 &55.6& 44.9& 69.1 &48.4 \\
                {FL-Swin-T\citep{flatten}} &49&268&46.5& 68.5 &50.8 &42.1& 65.4 & 45.1 \\
                {VMamba-T\citep{vmamba}} &50&270&48.9 &70.6& 53.6 &43.7 &67.7 &46.8 \\
                {MILA-T\citep{mlla}} &44&255&48.8& 71.0 &53.6& 43.8& 68.0 &46.8 \\

                \rowcolor{blue!10}
                {MirrorLA-S}&44&262& \textbf{51.3} & \textbf{72.3} & \textbf{56.4} & \textbf{45.4} & \textbf{69.6} & \textbf{49.0} \\

                \bottomrule[2pt]
            \end{tabular}
        }

        % \vspace{0.75ex}

        \captionof{table}{\textbf{Semantic segmentation results compared with SOTA methods on ADE20K and Cityscapes benchmarks.} FLOPs is calculated with input size 512$\times$512 for ADE20K and 2,048$\times$1,024 for Cityscapes.}
        \label{tab:seg}
        \setlength{\tabcolsep}{5pt}
        \resizebox{\linewidth}{!}{
            \begin{tabular}{l|c|cc|cc}
                \toprule[2pt]
                \multirow{2}{*}{\textsc{Method}} & \multirow{2}{*}{\textsc{Para}}&\multicolumn{2}{c|}{\textsc{ADE20K}}&\multicolumn{2}{c}{\textsc{CityScapes}}\\
                 & &\textsc{mIoU}&\textsc{FLOPs}&\textsc{mIoU}&\textsc{FLOPs}\\

                 \midrule
                 VWFormer-B1 \citep{vwformer}&14M &44.0&13G&80.4&-\\
                 EfficientViT-B2 \citep{efficientvit}& 15M& 45.9 &9.1G&82.1&74G\\
                 SegFormer-B1 \citep{segformer}&14M&42.2&16G &78.5&244G\\
                 SegNeXt-S \citep{segformer} &15M& 44.3& 16G&81.3&125G\\
                 \rowcolor{blue!10}
                 MirrorLA-T& 14M & \textbf{48.8} &14G & \textbf{82.5} & 110G\\

                 \midrule
                 VRWKV-S \citep{vrwkv}&29M&47.2&46G&-&-\\
                 MambaOut-T \citep{mambaout}&54M&47.4&-&-&-\\
                 VWFormer-B2 \citep{vwformer}&27M &48.1&47G&81.7&415G\\
                 EfficientViT-B3 \citep{efficientvit}& 40M& 49.0 &22G&83.0&179G\\
                 SegNeXt-B \citep{segnext} &28M &48.5& 35G&82.6 &276G\\
                 EfficientViT-L1\citep{efficientvit}& 40M& 49.2&36G& 82.7  &282G\\
                 SegFormer-B2 \citep{segformer} &28M& 46.5 &62G&81.0&711G\\
                 \rowcolor{blue!10}
                 MirrorLA-S& 24M & \textbf{50.9} & 25G  & \textbf{83.5} &  205G \\
                \bottomrule[2pt]
            \end{tabular}
        }
    \end{minipage}

\end{table}

\noindent\textbf{Results.}
The results in Tab.~\ref{tab:imagenet} demonstrate that MirrorLA consistently delivers strong and scalable performance across model sizes.
At the tiny scale, MirrorLA-T achieves 82.8\% Top-1 accuracy, improving upon RAVLT-T by 0.2\% and RMT-T by 0.4\%, indicating clear gains even in low-capacity regimes.
As model capacity increases, the performance advantage becomes more stable. MirrorLA-S reaches 84.2\%, surpassing RMT-S by 0.1\%, while MirrorLA-B further improves to 85.3\%, exceeding RAVLT-B by 0.2\%.
At the largest scale, MirrorLA-L attains 85.7\% Top-1 accuracy, outperforming both RMT-L and RAVLT-L, which highlights the effectiveness of MirrorLA in leveraging increased model capacity.
Overall, these results suggest that MirrorLA not only provides consistent improvements over prior methods but also scales favorably across different model sizes.

\subsection{Object Detection and Instance Segmentation}
\noindent\textbf{Settings.}
We evaluate MirrorLA on COCO \citep{coco} dataset for object detection and instance segmentation. By integrating MirrorLA into RetinaNet \citep{retinanet} and Mask R-CNN \citep{mrcnn} under both 1$\times$ and 3$\times$ training schedules. For fair comparison, we adopt the evaluation protocol of FL-Transformer \citep{flatten}.

\noindent\textbf{Results.}
As shown in Tab.~\ref{tab:COCO}, MirrorLA consistently improves detection and instance segmentation performance under comparable computational budgets.
Specifically, MirrorLA yields up to 1.8 $\textsc{AP}^b$ gains for RetinaNet, while for Mask R-CNN it improves $\textsc{AP}^b$ from 47.2 to 49.9 under the 1$\times$ schedule and from 50.7 to 51.3 under the 3$\times$ schedule.

\begin{figure*}[h]
    \centering
    \includegraphics[width=1\linewidth]{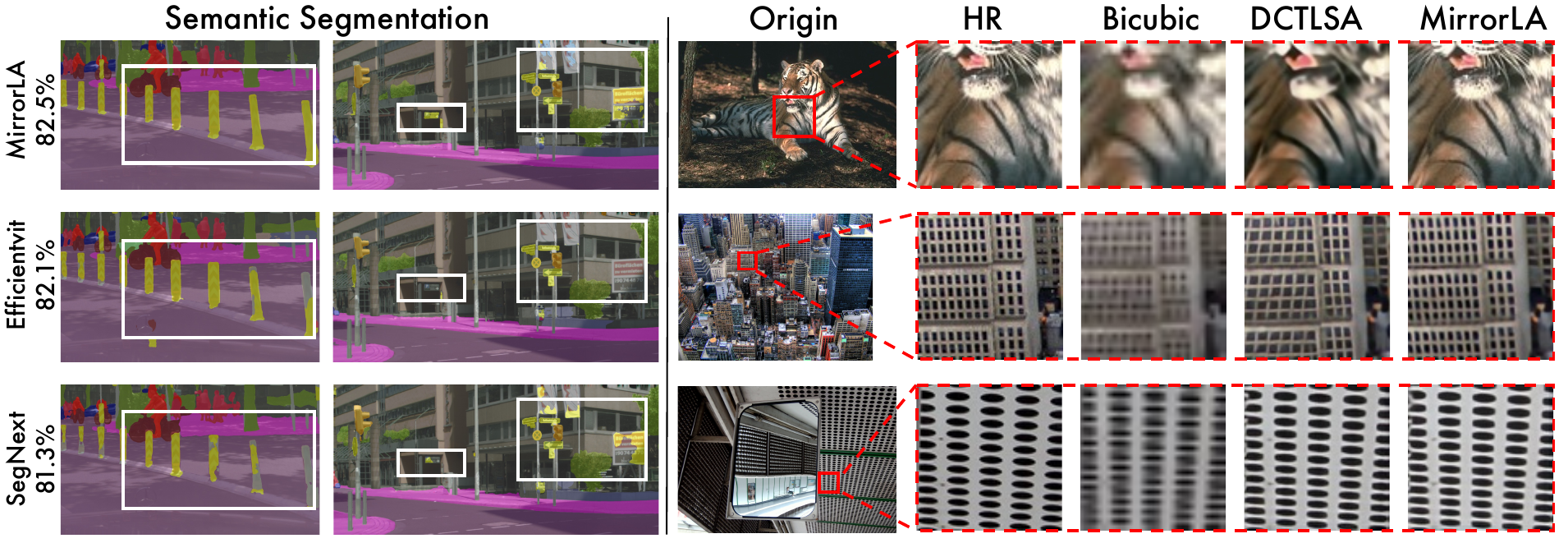}
    \caption{Visualization of Semantic Segmentation and Super-Resolution (SR) tasks. Left: Comparative results on the Cityscapes dataset, where MirrorLA achieves superior segmentation integrity, whereas competing methods suffer from incomplete masks. Right: Comparison of SR performance; MirrorLA reconstructs details more effectively, while DCTLSA introduces noticeable structural distortions.}
    \label{fig:segsr}
\end{figure*}

\subsection{Semantic Segmentation}
\noindent\textbf{Settings.}
We integrate MirrorLA into semantic segmentation on ADE20K \citep{ade20k} and Cityscapes \citep{cityscapes}, using mIoU as the evaluation metric. We initialize the backbone with ImageNet-1K pretrained weights and train the segmentation models using the \textit{mmcv-segmentation} framework \citep{mmcv}, following common practice in recent works \citep{segnext,efficientvit}.
Additionally, we adopt an FPN decoder with other linear attention methods in Appendix Tab.~\ref{tab:fpnseg}.

\noindent\textbf{Results.}
MirrorLA consistently raises mIoU under similar computational budgets on both datasets.
As shown in Tab.~\ref{tab:seg}, MirrorLA delivers stronger segmentation accuracy with competitive complexity. On ADE20K, MirrorLA-T and MirrorLA-S obtain 48.8 and 50.9 mIoU, outperforming EfficientViT-B2 \citep{efficientvit} at 45.9 and EfficientViT-L1 \citep{efficientvit} at 49.2. On Cityscapes, MirrorLA-T reaches 82.5 mIoU and MirrorLA-S further improves to 83.5, exceeding SegNeXt-S \citep{segnext} at 81.3 and EfficientViT-B3 \citep{efficientvit} at 83.0.
As shown in Fig.~\ref{fig:segsr} (left), MirrorLA produces noticeably cleaner and more precise semantic boundaries compared to EfficientViT and SegNext. It better separates adjacent objects, accurately captures small details such as traffic signs, and reduces misclassified or missing regions, leading to improved overall segmentation quality.

\subsection{Super Resolution}
\noindent\textbf{Settings.}
We evaluate MirrorLA on lightweight single-image super-resolution by integrating it into an efficient SR baseline, DCTLSA \citep{dctlsa}. Following common practice, we consider $\times 3$ and $\times 4$ upscaling and report PSNR/SSIM on standard SR benchmarks. We additionally measure inference latency on an RTX 3090 to assess efficiency.

\noindent\textbf{Results.}
Tab.~\ref{tab:lightSR} shows that integrating MirrorLA into DCTLSA consistently improves PSNR/SSIM for both $\times 3$ and $\times 4$ upscaling. For example, at $\times 4$, Set5 improves from 32.44/0.8973 to 32.53/0.8983, and Urban100 from 26.41/0.7944 to 26.65/0.8007. In addition, DCTLSA-MirrorLA significantly reduces latency and memory: on Manga109 $\times 4$, latency is reduced by over 75\% and memory usage is cut by over 80\%. As visualized in Fig.~\ref{fig:segsr} (right), MirrorLA produces super-resolved images closer to the ground truth, with clear details and better structural preservation.

\begin{table*}[h]
    \centering
    \caption{\textbf{Quantitative comparisons between MirrorLA and other lightweight image super-resolution methods.} \textsc{MEM.} denotes GPU memory consumption, and \textsc{LAT.} denotes latency (per image).}
    \label{tab:lightSR}

    \resizebox{1\linewidth}{!}{
        \begin{tabular}{l|c|cc|cc|cc|cc|cc}
        \toprule[2pt]
            \multirow{2}{*}{\textsc{Model}} &
            \multirow{2}{*}{\textsc{Scale}}&
            \multicolumn{2}{c|}{\textsc{Set5}} &
          \multicolumn{2}{c|}{\textsc{Set14}} &
          \multicolumn{2}{c|}{\textsc{BSDS100}} &
          \multicolumn{2}{c|}{\textsc{Urban100}}&
          \multicolumn{2}{c}{\textsc{Manga109}}
          \\
        &
        & {\textsc{PSNR}}  
        & {\textsc{SSIM}}  
        & {\textsc{PSNR}}  
        & {\textsc{SSIM}}  
        & {\textsc{PSNR}}  
        & {\textsc{SSIM}}  
        & {\textsc{PSNR}}  
        & {\textsc{SSIM}}  
        & {\textsc{PSNR}}  
        & {\textsc{SSIM}}  
        \\ 
        \midrule
        Bicubic & {$\times$4}
        & 28.42
        & 0.8104
        & 26.00
        & 0.7027
        & 25.96
        & 0.6675
        & 23.14
        & 0.6577
        & 24.89
        & 0.7866
        \\
        
        SwinIR~\citep{swinir}  & {$\times$4}
        & 32.44
        & 0.8976
        & 28.77
        & 0.7858
        & 27.69
        & 0.7406
        & 26.47
        & 0.7980
        & 30.92
        & 0.9151
        \\
        DCTLSA~\citep{dctlsa}  & {$\times$4}
        & 32.44
        & 0.8973
        & 28.77
        & 0.7846
        & 27.67
        & 0.7386
        & 26.41
        & 0.7944
        & 31.04
        & 0.9138
        \\
        \rowcolor{blue!10} 
        DCTLSA-MirrorLA  & {$\times$4}
        & \textbf{32.53}
        & \textbf{0.8983}
        & \textbf{28.84}
        & \textbf{0.7867}
        & \textbf{27.74}
        & \textbf{0.7409}
        & \textbf{26.65}
        & \textbf{0.8007}
        & \textbf{31.30}
        & \textbf{0.9170}
        \\

        \midrule
        
        Bicubic & {$\times$3}
        & 30.39 
        & 0.8682
        & 27.55 
        & 0.7742
        & 27.21
        & 0.7385
        & 24.46
        & 0.7349
        & 26.95
        & 0.8556
        \\
        SwinIR~\citep{swinir}  & {$\times$3}
        & 34.62
        & 0.9289
        & 30.54
        & 0.8463
        & 29.20
        & 0.8082
        & 28.66
        & 0.8624
        & 33.98
        & 0.9478
        \\
        DCTLSA~\citep{dctlsa}  & {$\times$3}
        & 34.76
        & 0.9296
        & 30.64
        & 0.8474
        & 29.27
        & 0.8093
        & 28.81
        & 0.8674
        & 34.42
        & 0.9492
        \\
        \rowcolor{blue!10} 
        DCTLSA-MirrorLA  & {$\times$3}
        & \textbf{34.84}
        & \textbf{0.9301}
        & \textbf{30.70}
        & \textbf{0.8488}
        & \textbf{29.33}
        & \textbf{0.8106}
        & \textbf{29.03}
        & \textbf{0.8688}
        & \textbf{34.61}
        & \textbf{0.9505}
        \\

        \midrule[1pt]

        \textsc{Efficiency@3090}
        & {\textsc{Scale}}
        & \textsc{Lat.}
        & \textsc{Mem.}
        & \textsc{Lat.}
        & \textsc{Mem.}
        & \textsc{Lat.}
        & \textsc{Mem.}
        & \textsc{Lat.}
        & \textsc{Mem.}
        & \textsc{Lat.}
        & \textsc{Mem.}
        \\ 
        \midrule
        % SwinIR \citep{swinir}     & {\textsc{Latency}}
        % & 72.51ms
        % & 101.34ms
        % & 109.22ms
        % & 166.26ms
        % & 88.81ms
        % & 114.46ms
        % & 354.39ms
        % & 672.67ms
        % & 451.03ms
        % & 855.38ms
        % \\
        DCTLSA-Softmax    & {\textsc{$\times4$}}
        & 57.70ms
        & 1.05GB
        & 109.04ms
        & 1.47GB
        & 72.06ms
        & 0.59GB
        & 349.86ms
        & 4.28GB
        & 456.71ms
        & 3.44GB
        \\

        DCTLSA-MirrorLA    & {\textsc{$\times4$}}
        & 55.28ms
        & 0.18GB
        & 55.29ms
        & 0.30GB
        & 49.62ms
        & 0.11GB
        & 79.03ms
        & 0.85GB
        & 100.69ms
        & 0.67GB
        \\
        \rowcolor{blue!10} 
        ~~~-~\textsc{Save} & {\textsc{$\times4$}}
        & \textbf{4.2\%}
        & \textbf{82.9\%}
        & \textbf{49.3\%}
        & \textbf{79.6\%}
        & \textbf{31.1\%}
        & \textbf{81.4\%}
        & \textbf{77.4\%}
        & \textbf{80.1\%}
        & \textbf{78.0\%}
        & \textbf{80.5\%}
        \\
        \midrule
        DCTLSA-Softmax    & {\textsc{$\times3$}}
        & 57.90ms
        & 1.66GB
        & 113.04ms
        & 2.27GB
        & 76.74ms
        & 0.93GB
        & 355.81ms
        & 7.19GB
        & 571.86ms
        & 5.49GB
        \\

        DCTLSA-MirrorLA    & {\textsc{$\times3$}}
        & 46.11ms
        & 0.33GB
        & 51.82ms
        & 0.51GB
        & 42.70ms
        & 0.19GB
        & 146.27ms
        & 1.52GB
        & 183.47ms
        & 1.20GB
        \\
        \rowcolor{blue!10} 
        ~~~-~\textsc{Save} & {\textsc{$\times3$}}
        & \textbf{20.4\%}
        & \textbf{80.1\%}
        & \textbf{54.2\%}
        & \textbf{77.5\%}
        & \textbf{44.3\%}
        & \textbf{79.6\%}
        & \textbf{58.9\%}
        & \textbf{78.9\%}
        & \textbf{67.9\%}
        & \textbf{78.1\%}
        \\

        % \rowcolor{blue!10} 
        % ~~~-~\textsc{Speedup} & {\textsc{Latency}}
        % & \textbf{23.8\%}
        % & \textbf{54.5\%}
        % & \textbf{49.4\%}
        % & \textbf{74.9\%}
        % & \textbf{44.1\%}
        % & \textbf{62.7\%}
        % & \textbf{77.7\%}
        % & \textbf{78.3\%}
        % & \textbf{77.8\%}
        % & \textbf{78.6\%}
        % \\
        
        \bottomrule[2pt]
        \end{tabular}
    }
\end{table*}

\subsection{Diffusion Transformer}
\noindent\textbf{Settings.}
Diffusion transformers provide a suitable testbed for evaluating linear attention. Following DiT \citep{dit} and SiT \citep{sit}, we replace the attention module with MirrorLA and conduct experiments on ImageNet-1K \citep{imagenet} under the same training and evaluation protocols as the corresponding baselines.

\noindent\textbf{Results.}
Tab.~\ref{tab:dit} reports that integrating MirrorLA consistently improves diffusion transformers across both DiT and SiT architectures.
On the DiT branch, MirrorLADiT yields a sizable improvement over DiG \citep{dig}. Specifically, the FID decreases from 62.06 to 58.07 and the sFID decreases from 11.77 to 11.21, while the IS increases from 22.81 to 24.70.
On the SiT branch, MirrorLASiT achieves the best overall performance. It attains lower FID and sFID values of 52.29 and 8.41, respectively, together with a higher IS of 28.51, and demonstrates strong precision and recall of 0.43 and 0.61. This performance surpasses both SiT \citep{sit} and EfficientSiT \citep{efficientsit}.

\begin{table}[h]
        \centering
        \caption{\textbf{Diffusion transformer results.} MirrorLA achieves outstanding performance on both DiT and SiT.}
        % \setlength{\tabcolsep}{2pt}
        % \resizebox{1\linewidth}{!}{
        \begin{tabular}{l|ccccc}
        \toprule[2pt]
         \textsc{Model}&\textsc{FID} $\downarrow$&\textsc{sFID} $\downarrow$&\textsc{IS} $\uparrow$&\textsc{Precision} $\uparrow$&\textsc{Recall} $\uparrow$ \\
         \midrule
            DiT \citep{dit}   & 68.40&-&-&-&-\\
            DiG \citep{dig}   &62.06&11.77&22.81&0.39&0.56\\
            \rowcolor{blue!10}
            MirrorLADiT &\textbf{58.07}&\textbf{11.21}&\textbf{24.70}&\textbf{0.41}&\textbf{0.59}\\
            \midrule[1pt]
            SiT \citep{sit}   & 58.61&9.25&24.31&0.41&0.59\\
            EfficientSiT \citep{efficientsit} &53.57&9.01&27.26&0.43&0.61\\
            \rowcolor{blue!10}
            MirrorLASiT &\textbf{52.29}&\textbf{8.41}&\textbf{28.51}&\textbf{0.43}&\textbf{0.61} \\
          \bottomrule[2pt]
        \end{tabular}
        % }
        \label{tab:dit}
    \end{table} 

\subsection{Ablation study}
We evaluate MirrorLA on the Swin-T architecture following the protocol of FLatten-Transformer \citep{flatten} and compare it with existing linear attention mechanisms.
As a modular enhancement inserted before the non-negativity constraint, our method surpasses existing linear attention variants with a 0.6\% accuracy gain shown in Appendix Tab. \ref{tab:swin}. 
Additionally, we conduct extensive ablation studies further confirm the necessity of our hierarchical design: on the XT setting, disabling the full variant or removing cross-head reflection leads to performance drops of 0.5\% and 0.3\%, respectively. 
Detailed analyses regarding hyperparameter sensitivity ($\lambda, \alpha_{\max}$) and component-wise performance are provided in the Appendix Tab. \ref{tab:ablation}.

\section{Conclusion}
In this work, we introduce \textbf{MirrorLA}, a linear attention framework that replaces passive non-negativity clipping with active geometric reorientation via learnable Householder reflections. 
At the single-head level, we generalize 2D reflections to high-dimensional feature spaces through block-wise decomposition, and further propose Variance-Aware Angle Modulation to enable data-dependent mirror transformations.
At the global scale, we design a unified cross-head reflection to enhance inter-head communication. 
Extensive experiments demonstrate that MirrorLA consistently achieves superior performance while substantially reducing memory footprint and time.

\clearpage

\bibliographystyle{plainnat}
\setlength{\bibhang}{0pt}
\setlength\bibindent{0pt}
\bibliography{main}

\newpage
\appendix

\section{Appendix}
\begin{itemize}
    % \item \ref{appendix:llmusage} \textbf{LLM Usage Statement.}
    \item \ref{app:modelsettings} \textbf{Model Settings.}
    \item \ref{app:ablation} \textbf{Ablation Study.}
    \item \ref{app:visualizations} \textbf{Visualizations.}
    \item \ref{app:discussion} \textbf{Discussion.} Efficient Implement Discussion.
    % \item \ref{sec:DiT} \textbf{Diffusion Transformer Results.} DiT experiments.
\end{itemize}

\subsection{Model Settings}
Consistent with established works \citep{rala, rmt, swin}, we develope a set of Mirror-LA backbones by only replace the non-negative activations with our reflecting feature map, each with varying configurations of block count and channel dimensions across their respective stages whose the ratio of MLP is set to 3.5. The architecture details are illustrated in the Tab. \ref{tab:modeldetails}.
\label{app:modelsettings}
\begin{table}[h]
\centering
\caption{\textbf{Architecture setting details of MirrorLA.}}
\label{tab:modeldetails}
% \resizebox{1\linewidth}{!}{
\begin{tabular}{c|cc|ccc}
\toprule[2pt]
\textsc{Model} &\textsc{Param}& \textsc{Flops}&  \textsc{Blocks} & \textsc{Channels} & \textsc{Heads}\\
\midrule[1pt]
MirrorLA-XT         &6M &  0.6G       & {[}2, 2, 2, 2{]}           & {[}32, 64, 128, 384{]}      & {[}1, 2, 4, 8{]}           \\
MirrorLA-T         &15M   &  2.6G       & {[}2, 2, 6, 2{]}           & {[}64, 128, 256, 512{]}   &    {[}1, 2, 4, 8{]}           \\
MirrorLA-S   &26M      &  4.9G        & {[}3, 5, 9, 3{]}           & {[}64, 128, 320, 512{]}   &    {[}1, 2, 5, 8{]}          \\
MirrorLA-B    &48M     &  10.6G         & {[}4, 6, 12, 6{]}          & {[}96, 192, 384, 512{]}     &  {[}1, 2, 6, 8{]}         \\
MirrorLA-L   &95M    &  17.3G           & {[}4, 7, 19, 8{]}          & {[}96, 192, 448, 640{]}      & {[}1, 2, 7, 10{]}         \\
\bottomrule[2pt]
\end{tabular}
% }
\end{table}

% \clearpage
\subsection{Ablation Study}
\label{app:ablation}
\noindent\textbf{Comparison with Other Linear Attention.}
To ensure a rigorous and fair comparison, we adhere to the evaluation protocol established by FLatten-Transformer  using the Swin-T architecture. MirrorLA is designed as a modular enhancement: we simply insert our reflecting feature map operation immediately before the non-negativity constraint, while keeping all other components—including specific "spiky" activation functions—entirely unchanged. Results are shown in Tab. \ref{tab:swin}

\begin{table}[h]

\centering
        \caption{\textbf{Comparison with other linear attention models on the Swin-T setting.}}
        % \resizebox{\linewidth}{!}{
        \begin{tabular}{l|ccc}
            \toprule[2pt]
            \textsc{Method} & \textsc{Params} & \textsc{FLOPs} & \textsc{Acc. (\%)} \\
            \midrule
            Swin-T \citep{swin} & 28M & 4.4G & 81.2 \\
            Hydra Attn \citep{hydra} & 29M & 4.5G & 80.7 \\
            Efficient Attn \citep{efficientattn} & 29M & 4.5G & 81.0 \\
            Linear Angular \citep{angularattn} & 29M & 4.5G & 79.4 \\
            FLatten Attn \citep{flatten} & 29M & 4.5G & 82.1 \\
            Agent Attn \citep{agentattn} & 29M & 4.5G & 82.6 \\
            {InLine Attn} \citep{inline} & 30M & 4.5G & 82.4\\
            PolaFormer \citep{polaformer} & 29M & 4.5G & 82.6 \\
            \midrule
            ReLU+LA & 28M & 4.6G  & 81.8 \\
            \rowcolor{blue!10}
            ReLU+MirrorLA & 28M & 4.6G & \textbf{82.4} \\
            \bottomrule[2pt]
        \end{tabular}
        % }
        \label{tab:swin}

\end{table}

\noindent\textbf{Hyperparameters and components Ablations.}
We study the impact of key design choices in MirrorLA on the XT setting. Specifically, we vary the $\lambda$ used in $\operatorname{sigmoid}\left(\frac{\lambda}{\sigma_m^2+\varepsilon}\right)$, the maximum angular perturbation $\alpha_{\max}$ , and whether to enable the full variant and the cross-head reflection. As shown in Tab.~\ref{tab:ablation}, the default setting achieves the best accuracy. Disabling the full variant causes a 0.5 drop, while removing cross-head reflection reduces accuracy by 0.3.

% \clearpage
\begin{table}[h]
    \centering
    \caption{\textbf{Hyperparameter ablations} for MirrorLA.}
    % \resizebox{0.5\linewidth}{!}{
    \begin{tabular}{cccc|l}
        \toprule[2pt]
        \textsc{$\lambda$} &\textsc{$\alpha_{\max}$} & \textsc{full} & \textsc{head} & \textsc{Acc. (\%)} \\
        \midrule
        1 & $\pi/2$ & $\checkmark$ & $\checkmark$ & 76.0 \\
        1 & $\pi/2$ &  & $\checkmark$ & 75.5$_\text{\textcolor{red}{-0.5}}$ \\
        1 & $\pi/2$ & $\checkmark$ &  & 75.7$_\text{\textcolor{red}{-0.3}}$ \\
        \midrule
        0.5 & $\pi/2$ & $\checkmark$ & $\checkmark$ & 75.7$_\text{\textcolor{red}{-0.3}}$ \\
        0.25 & $\pi/2$ & $\checkmark$ & $\checkmark$ & 75.8$_\text{\textcolor{red}{-0.2}}$ \\
        1 & $\pi/4$ & $\checkmark$ & $\checkmark$ & 75.9$_\text{\textcolor{red}{-0.1}}$ \\
        1 & $\pi/8$ & $\checkmark$ & $\checkmark$ & 75.8$_\text{\textcolor{red}{-0.2}}$ \\
        \bottomrule[2pt]
    \end{tabular}
    % }
    \label{tab:ablation}
\end{table}

\noindent\textbf{More Results.}
We additionally conduct semantic segmentation experiments using the MMSegmentation framework \citep{mmcv}. 
Following the same settings as PVT \citep{pvt}, we train models with a Semantic FPN decoder.
Under the same model scale, MirrorLA achieves the best performance.

\begin{table}[h]
    \centering
    
    \caption{\textbf{More semantic segmentation results.} Results on Semantic FPN 80K.}
    % \label{tab:seg}
    % \resizebox{1\linewidth}{!}{
        \begin{tabular}{l|cc|c}
            \toprule[2pt]
            \textsc{Method} & {\textsc{Para}}&\textsc{FLOPs}&\textsc{mIoU}\\
            \midrule
            VAN-B1 \citep{van}&18M&140G&42.9\\
            PVT-v2-B1 \cite{pvtv2}&18M&136G&42.5\\
            RMT-T \citep{rmt}&17M&136G&46.4\\
            FL-PVT-T \citep{flatten}&16M&169G&37.2\\
            Agent-PVT-T \citep{agentattn} &15M&147G&40.2\\
            % Pola-PVT-T \cite{polaformer}&-&-&38.3\\
            RAVLT-T \citep{rala}&18M&136G&47.5\\
            \rowcolor{blue!10}
            MirrorLA-T&18M&136G&\textbf{47.7}\\

            \bottomrule[2pt]
        \end{tabular}
        \label{tab:fpnseg}
    % }
\end{table}

\clearpage
\subsection{Visualizations}
\label{app:visualizations}
\begin{figure}[!h]
    \centering
    \includegraphics[width=0.9\linewidth]{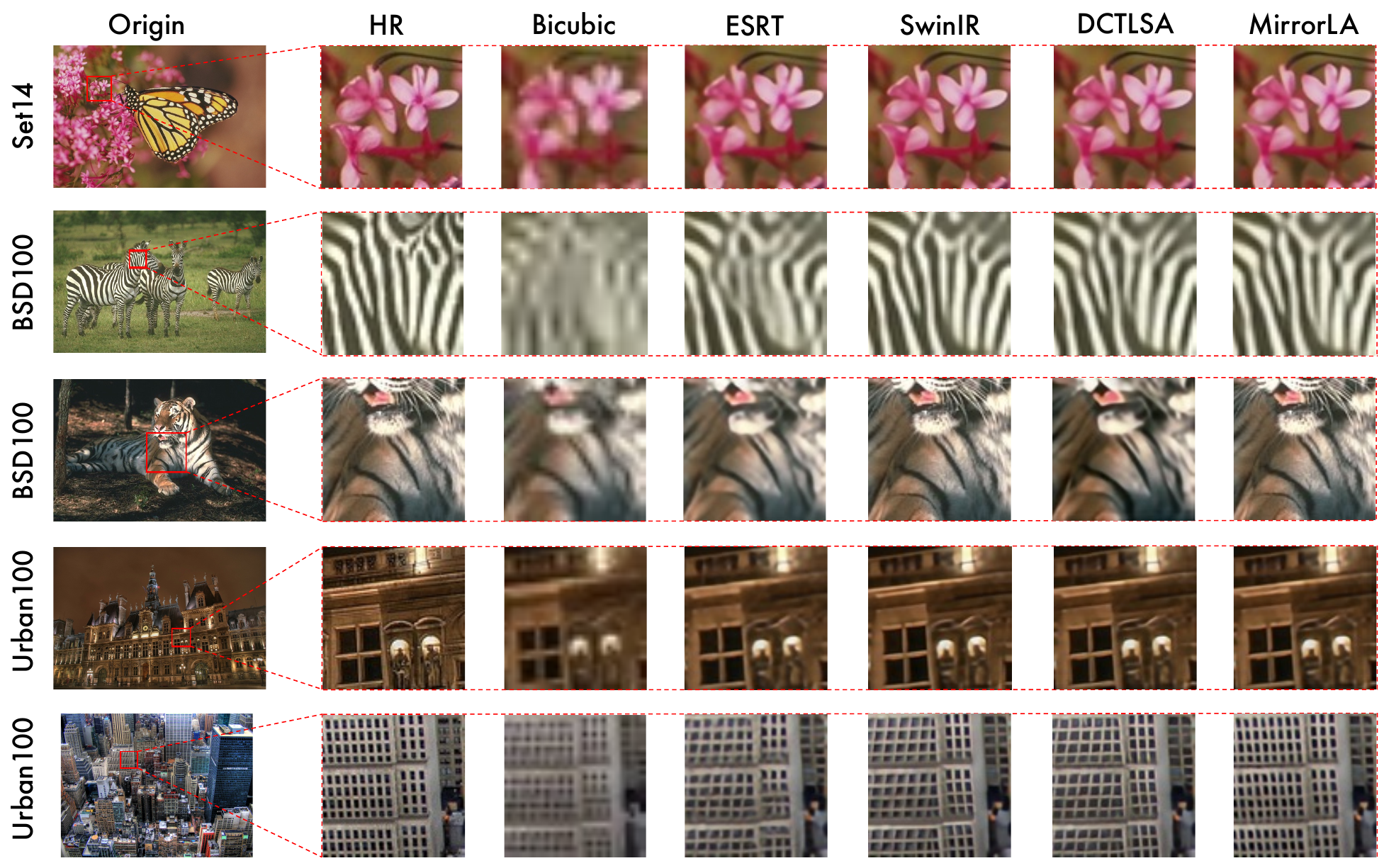}
    \caption{More visualizations on super-resolution tasks. Our model produces more clear boundaries, more accurate shapes, and finer-grained textures compared to baselines.}
    % \label{fig:placeholder}
\end{figure}

\begin{figure}[!h]
    \centering
    \includegraphics[width=0.9\linewidth]{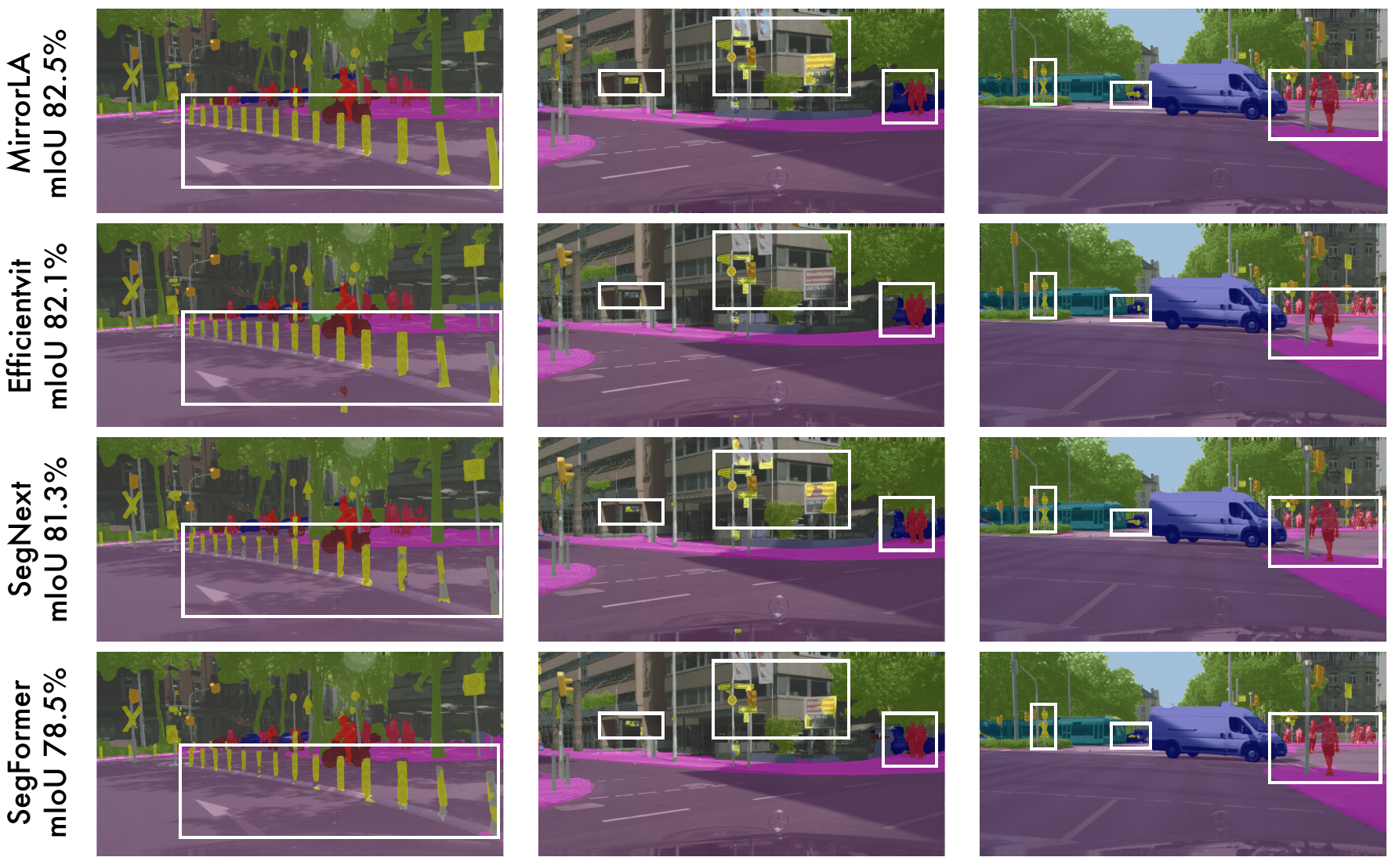}
    \caption{Visualization results on semantic segmention. Compared to other methods, our approach more accurately delineates different objects, effectively avoiding missing predictions and boundary ambiguities.}
    % \label{fig:placeholder}
\end{figure}

\subsection{Discussion}
\label{app:discussion}
\noindent\textbf{Future Work: Hardware-Aware Design.}
The practical utility of linear attention models increasingly relies on high-performance acceleration libraries. A prominent example is Flash Linear Attention (FLA), a specialized Triton-based library designed for the efficient computation of linear attention. While MirrorLA demonstrates superior representational capacity, its full integration with such high-performance hardware kernels presents an intriguing avenue for future research. Currently, the implementation of our geometric modulations, such as Householder reflections, is primarily optimized at the functional level.

In future work, we aim to implement block-wise architectural optimizations to better align with the parallel processing logic of modern GPUs. Specifically, we plan to develop dedicated Triton-based kernels to fuse these geometric operations directly into the FLA framework. Such hardware-aware synergy would be particularly beneficial for high-dimensional 2D vision tasks, where custom tiling strategies could mitigate register pressure and SRAM constraints. By bridging the gap between active geometric modeling and low-level operator fusion, we believe MirrorLA can achieve a more optimal balance between theoretical expressivity and practical inference throughput across diverse modalities.

\end{document}